\documentclass[12pt]{l4dc2021} 

\usepackage{color}
\usepackage{amsfonts}
\usepackage{graphicx}
\usepackage{dsfont}
\usepackage{psfrag}
\usepackage{mathrsfs}
\usepackage{epstopdf}
\usepackage{mathtools}
\usepackage{enumitem}
\usepackage{lipsum}

\makeatletter
\def\set@curr@file#1{\def\@curr@file{#1}} 
\makeatother
\usepackage[load-configurations=version-1]{siunitx} 

\newtheorem{assumption}{Assumption}
\newtheorem*{assumption*}{Assumption}

\DeclareMathOperator*{\argmax}{arg\,max}
\DeclareMathOperator*{\argmin}{arg\,min}

\usepackage{tikz-cd}

\title[]{A Dynamical Systems Stability Approach for\\ Convergence of the Bayesian EM Algorithm}
\usepackage{times}

\author{%
 \Name{Orlando Romero} \Email{oromero@seas.upenn.edu}\\
 \addr Department of Electrical and Systems Engineering, University of Pennsylvania, Philadelphia, PA, USA
 \AND
 \Name{Subhro Das} \Email{subhro.das@ibm.com}\\
 \addr MIT-IBM Watson AI Lab, IBM Research, Cambridge, MA, USA
 \AND\Name{Pin-Yu Chen} \Email{pin-yu.chen@ibm.com}\\
 \addr IBM Thomas J. Watson Research Center, IBM Research, Yorktown Heights, NY, USA
 \AND\Name{S\'{e}rgio Pequito} \Email{sergio.pequito@tudelft.nl}\\
 \addr Delft Center for Systems and Control, Delft University of Technology, Delft, The Netherlands
}

\begin{document}

\maketitle

\begin{abstract}%
Out of the recent advances in systems and control (S\&C)-based analysis of optimization algorithms, not enough work has been specifically dedicated to machine learning (ML) algorithms and its applications. This paper addresses this gap by illustrating how (discrete-time) Lyapunov stability theory can serve as a powerful tool to aid, or even lead, in the analysis (and potential design) of optimization algorithms that are not necessarily gradient-based. The particular ML problem that this paper focuses on is that of parameter estimation in an incomplete-data Bayesian framework via the popular optimization algorithm known as maximum a posteriori expectation-maximization (MAP-EM). Following first principles from dynamical systems stability theory, conditions for convergence of MAP-EM are developed. Furthermore, if additional assumptions are met, we show that fast convergence (linear or quadratic) is achieved, which could have been difficult to unveil without our adopted S\&C approach. The convergence guarantees in this paper effectively expand the set of sufficient conditions for EM applications, thereby demonstrating the potential of similar S\&C-based convergence analysis of other ML algorithms.
\end{abstract}


\section{Introduction}
This work builds upon~\citep{romero2019} and is inspired by recent papers importing ideas from (dynamical) systems and control (S\&C) theory into optimization, such as~\citep{Elia2011,Su2014,Lessard2016,Wibisono2016,Fazlyab2017,Scieur2017,Franca2018,Taylor2018,Wilson2018,Orvieto2019, Romero2020}. While a significant volume of these optimization-based papers have been published at machine learning (ML) venues, only a few have been explicitly dedicated to addressing concrete ML problems, applications, or algorithms~\citep{Plumbley1995,Pequito2011UnsupervisedLO,Zhu2018AnOC,Aquilanti2019,Liu2019}. Furthermore, only a small subset of this emerging topic of research has focused directly on discrete-time analysis that is the direct result of discretizations of an underlying continuous-time version of the algorithms~\citep{Lessard2016,Fazlyab2018,Fazlyab2018b,Taylor2018,Lessard2020}.

Lyapunov stability theory, extensively used to analyze the stability of nonlinear dynamical systems~\citep{Khalil2001}, is a particularly fruitful approach to import from S\&C into O\&ML. In general, Lyapunov functions may be seen as \emph{abstract} surrogates of \emph{energy} in a dynamical system. If such a function is sufficiently regular and non-increasing over time, then some form of stability must be present. Likewise, if persistently increasing, then instability is inevitable. However, in general, constructing a suitable can be a difficult endeavour. Fortunately, in the context of O\&ML, the cost function itself (if available) or other available performance metrics are often good candidates or starting points to constructing useful Lyapunov functions. This way, parallels between notions o \emph{stability} of dynamical systems and \emph{convergence} of machine learning algorithms can be made, particularly so for non-combinatorial optimization-based algorithms. To the best of the authors' knowledge, the current literature lacks a comprehensive summary of these relationships, with the closest work that we are aware being from~\cite{Schropp1995,Lessard2016,Fazlyab2018,Taylor2018}, and~\cite{Wilson2018}. 

In this paper, we conduct a S\&C-based analysis of the convergence of a widely popular algorithm used for incomplete-data estimation and unsupervised learning -- the expectation-maximization (EM) algorithm. More precisely, we focus on the Bayesian variant of the EM algorithm originally proposed by~\cite{Dempster1977}, which we refer to as the \emph{MAP-EM algorithm}, since it is used for maximum a posteriori (MAP) estimation~\citep{Figueiredo2004}. 
We leverage notions from discrete-time Lyapunov stability theory to study the convergence of MAP-EM, and, in the process, provide exclusive insights on the robustness of our derived conditions for stability (asymptotic or otherwise), and thus convergence guarantees. 
%
%

Compared to our preliminary work~\citep{romero2019}, the present paper now allows us to incorporate arbitrary prior information on the unknown parameters to be estimated, thus potentially accelerating convergence, or otherwise improving its quality. Furthermore, we now provide less restrictive conditions to check to ensure different forms of convergence, particularly so for exponential stability, and thus Q-linear convergence of the iterates of the EM algorithm.
With this, we argue for the possibility of extending our S\&C-based framework to discover robust stability conditions and novel convergence guarantees of alternative iterative optimization algorithms used in machine learning.
%
%
%
%

\section{Background: MAP-EM Algorithm}

Let $\theta$ be an unknown parameter of interest that we seek to infer from an idealized unobservable dataset $x$, via the statistical model $p(x|\theta)$ and prior $p(\theta)$. Given that $x$ is not directly observable, suppose another random variable $y$, which may be seen as an \emph{incomplete} version of~$x$, is observable. In this definition, $x = (y,z)$, with~$z$ seen as \mbox{\emph{missing data}}. For this reason, $x$ is typically referred to as the \emph{complete} dataset. More generally, we could have $x = g(y,z)$ for some $g(\cdot)$, with $y$ observable and $z$ hidden, or, simplistically $y = h(x)$ for some $h(\cdot)$. In practice, there could even be no explicit missing data or no relationship between $x$ and $y$ in terms of transformations. Instead the only relationship could be the Markov condition $\theta\to x\to y$~\citep{Gupta2011}, meaning that~$y$ is conditionally independent of~$\theta$ subject to~$x$, \emph{i.e.} $p(y|x,\theta) = p(y|x)$.

We adopt the notation that $x,y,\theta$ are all (absolutely) continuous random variables, but, in reality, only $\theta|y$ is required to be so, with $x$ and $y$ being allowed to be discrete, continuous, or of mixed type. For ease of notation, we use $\mathrm{d}x,\mathrm{d}y,\mathrm{d}\theta$ to refer, respectively, to integration with respect to (w.r.t.) implicit $\sigma$-finite measures that dominate the probability distributions of $x,y,\theta$, or appropriate conditionals of these, and which coincide with the construction of the respective densities via Radon-Nikodym derivatives.
We aim to estimate $\theta$ from $y$ via the \emph{maximum a posteriori} (MAP) estimator: $\hat{\theta}_\textnormal{MAP} \triangleq \argmax_\theta\, \log p(\theta|y)$,
where the maximization is taken over the entire parameter space. In some situations a global maximizer may not exist, and we need to be content with (or even give preference to) a ``good'' local maximizer~\citep{Figueiredo2004}. The mapping $\theta\mapsto \log p(\theta|y)$ is  typically referred to as the \emph{incomplete-data} log-likelihood function, whereas \mbox{$\theta\mapsto \log p(\theta|x)$} the \emph{complete-data} log-likelihood function. 

Under the incomplete-data framework, a \mbox{popular} approach to compute the MAP estimator is through the \emph{expectation-maximization} (EM) algorithm (also known as MAP-EM in our Bayesian framework), whose iterations are
\begin{equation}
    \label{eq:EMalg}
    \hat{\theta}_{k+1} = \argmax_\theta \,Q(\theta,\hat{\theta}_k), \quad\quad (k\in\mathbb{Z}_+)
\end{equation}
from a given initial estimate $\hat{\theta}_0$, where $Q(\theta,\hat{\theta}) \triangleq \mathbb{E}_{x\sim p(x|y,\hat{\theta})}[\log p(\theta|x)]$ denotes the expected complete-data log-posterior, conditional to the observed data $y$ and current parameter estimate~$\hat{\theta}$. The maximum in~\eqref{eq:EMalg} is taken within the \emph{entire} parameter space. Eventually, as we will demonstrate later, EM is (in general) a \emph{local} (greedy) search method w.r.t. the actual objective function -- the (incomplete-data) log-posterior
\footnote{\footnotesize The underlying assumption of the EM algorithm is that~\eqref{eq:EMalg} is easy to \emph{globally} maximize in $\theta$. If it can't be exactly and globally maximized, but instead if we can find some~$\hat{\theta}_{k+1}$ such that $Q(\hat{\theta}_{k+1},\hat{\theta}_k) > Q(\hat{\theta}_k,\hat{\theta}_k), \forall \theta,$ (potentially excluding the case when~$\hat{\theta}_k$ is already fixed point of~EM), then any variant of EM that settles for such a sequence~$\{\hat{\theta}_k\}_{k\in\mathbb{Z}_+}$ of iterates, where $\mathbb{Z}_+ = \{0,1,2,\ldots\}$, is known as a \emph{generalized} EM (GEM) algorithm.}. 

\subsection{An Information-Theoretic Perspective}
\label{subsec:info_theory}
As we dive into the information theoretic perspective, recall that $\mathcal{D}_\mathrm{KL}(p\|q) \triangleq \int_\mathcal{X} p(x)\log\left(\frac{p(x)}{q(x)}\right)\mathrm{d}x$
denotes the Kullback-Leibler (KL) divergence between the probability density functions $p(x)$ and $q(x)$ w.r.t. the same $\sigma$-finite dominating measure over~$\mathcal{X}$ implicitly denoted via~$\mathrm{d}x$. The (random) variables of interest satisfies the following  assumption
\begin{assumption}
The random variables~$x,y,\theta$ satisfy the Markovian condition $\theta\to x\to y$.
\label{ass:Markov}
\end{assumption}
Now, the Markov Assumption~\ref{ass:Markov}, \emph{i.e.} $p(y|x,\theta) = p(y|x)$, allows us to restate the $Q$-function and subsequently the EM algorithm in information-theoretic terms, leading to the following proposition.
\begin{proposition}
Under Assumption~\ref{ass:Markov}, we have
\begin{equation}
\begin{split}
    Q(\theta,\hat{\theta}) = \log p(\theta|y) - \mathcal{D}_\mathrm{KL}[p_X(\cdot|y,\hat{\theta})\,\|\, p_X(\cdot|y,\theta)] +\textnormal{terms that do not depend on } \theta,
\end{split}
\label{eq:Qtemp}
\end{equation}
and, thus, the MAP-EM algorithm~\eqref{eq:EMalg} can be rewritten as
\begin{equation}
    \hat{\theta}_{k+1} = \argmax_\theta\,\{\log p(\theta|y)- d(\theta,\hat{\theta}_k)\}, 
\label{eq:EMnew}
\end{equation}
where $d(\theta,\hat{\theta}) \triangleq \mathcal{D}_{\mathrm{KL}}[p_X(\cdot|y,\hat{\theta})\,\|\,p_X(\cdot|y,\theta)]$ and $p_X(x|y,\theta) \triangleq p(x|y,\theta)$.
\label{prop:infoEM}
\end{proposition}

Since the sole purpose of the $Q$-function is to be maximized at the M-step, we thus redefine it as $Q(\theta,\hat{\theta}) \triangleq  \log p(\theta|y) - d(\theta,\hat{\theta})$ by dropping the terms that do not depend on~$\theta$. Notice that the EM algorithm can be recognized as a (generalized) proximal point algorithm (PPA)
\begin{equation}
\label{eq:EMcompact}
    \hat{\theta}_k = \argmin_\theta\,\{\ell(\theta) + \beta_k\,d(\theta,\hat{\theta})\},
\end{equation}
with a fixed step size $\beta_k = 1$, where the function we seek to minimize~$\ell(\theta) \triangleq -\log p(\theta|y)$. Furthermore, $d(\cdot,\cdot)$ may be seen as a regularizing premetric\footnote{\footnotesize It is well known that $d(\theta,\hat{\theta}) \geq 0$, with equality corresponding almost exclusively to $\theta=\hat{\theta}$. More precisely, $\mathcal{D}_\mathrm{KL}(p\|q) \geq 0$, with equality if and only if $p=q$ $P$-a.s., where $P$ is the (unique) probability measure with density (Radon–Nikodym derivative) $p$ w.r.t. the $\sigma$-finite measure implicit in $\mathrm{d}x$. Despite this, the KL divergence is non-symmetric and does not satisfy the triangle inequality, in general, and thus it is not a metric.}. This perspective was explored in detail by~\cite{Chretien2000} and~\cite{Figueiredo2004}, which served as inspiration for this work.

Notice also that, if the prior $p(\theta)$ is ``flat'', meaning that we assign the (typically) non-informative (and potentially degenerate) prior given by a uniform distribution over the parameter space, then \mbox{$\ell(\theta)\propto -\log p(y|\theta)$}, with $\theta\mapsto \log p(y|\theta)$ naturally denoting the incomplete-data log-likelihood function. Therefore, the Bayesian EM algorithm (MAP-EM) generalizes the traditional EM algorithm, and for this reason, from this point on we will largely refer to MAP-EM as simply EM.

\section{The Dynamical Systems Approach}
We now interpret the EM algorithm as a (discrete-time and time-invariant) nonlinear state-space dynamical system,
\begin{equation}
    \label{eq:dynamicalsystem}
    \hat{\theta}_{k+1} = F(\hat{\theta}_k), \quad\quad (k\in\mathbb{Z}_+)
\end{equation}
with $ F(\hat{\theta}) \triangleq  \argmax_\theta Q(\theta,\hat{\theta})$. In the language of dynamical systems, $\hat{\theta}_k$ is known as the \emph{state} of the system~\eqref{eq:dynamicalsystem} at time step $k$ (usually denoted as $x_k$, much like in reinforcement learning, but for the sake of notational consistency with EM we opted for $\hat{\theta}_k)$, and, in particular, $\hat{\theta}_0$ is called the \emph{initial state}. The space of points from which the state can take values is known as the \emph{state space}. In our case, the parameter space and the state space coincide. The sequence $\{\hat{\theta}_0,\hat{\theta}_1,\ldots\}$ is often called a \emph{trajectory} starting from initial state $\hat{\theta}_0$. The function $F$ represents the \emph{dynamics} of the system. The following assumption ensures that $F(\hat{\theta})$ is uniquely defined.
\begin{assumption}
$Q(\cdot,\hat{\theta})$ has a unique global maximizer for each $\hat{\theta}$.
\label{ass:Quniquemaximizer}
\end{assumption}
In other words, the \emph{complete}-data log-posterior is expected to have a unique global maximizer, which in principle does not prevent the \emph{incomplete}-data log-posterior from having multiple global maxima or being unbounded, such as in the cases of GMMs with unknown covariance matrices. 

\subsection{Equilibrium in Dynamical Systems}

We say that a point $\hat{\theta}^\star$ in the state space is an~\emph{equilibrium} of the dynamical system~\eqref{eq:dynamicalsystem} if $\hat{\theta}_0 = \hat{\theta}^\star$ implies that $\hat{\theta}_k = \hat{\theta}^\star$ for every $k\in\mathbb{Z}_+$. In other words, $\hat{\theta}^\star$ is an equilibrium point if and only if $\hat{\theta}^\star$ is a \emph{fixed point} of $F$. This implies that, $F(\hat{\theta}^\star) = \hat{\theta}^\star$, the equilibria of the dynamical system representation of EM naturally coincide with its fixed points.

Recall that \emph{limit points} of the EM algorithm~\eqref{eq:EMalg} consist of points $\hat{\theta}^\star$ for which there exists some $\hat{\theta}_0$ such that \mbox{$\hat{\theta}_k \to \hat{\theta}^\star$} as $k\to\infty$. However, EM need not be \emph{locally convergent} near limit points of its dynamical system representation, which requires that $\hat{\theta}_k \to \hat{\theta}^\star$ as $k\to\infty$ for every~$\hat{\theta}_0$ in a small enough neighborhood of~$\hat{\theta}^\star$. Such points are known as \emph{(locally) attractive} in dynamical systems theory.
In order to establish a key relationship between limit points of EM and equilibria of its dynamical systems representation, $F$ needs to be continuous as stated in Assumption~\ref{ass:Qcontinuous}.
\begin{assumption}
The $Q$-function $Q(\cdot)$ is continuous in both arguments.
\label{ass:Qcontinuous}
\end{assumption}
This follows, for instance, if $\theta\mapsto p(\theta|y)$ is continuous and $x$ conditional to $y$ and $\theta$ has a finite support, which is the case for GMM clustering.
\begin{lemma}
Under Assumptions~\ref{ass:Markov}--\ref{ass:Qcontinuous}, $F$ is continuous.
\label{lemma:Fiscontinuous}
\end{lemma}
Please refer to Appendix~\ref{sec:lemma2proof} for the proof of~Lemma~\ref{lemma:Fiscontinuous}. We now move ahead to establish a key relationship between limit points of EM and its dynamical systems representation.

\begin{proposition}
Under Assumptions~\ref{ass:Markov}--\ref{ass:Qcontinuous}, any limit point of the EM algorithm~\eqref{eq:EMalg} is also a fixed point of $F$ given by~\eqref{eq:dynamicalsystem}, and thus an equilibrium of EM's dynamical system representation~\eqref{eq:dynamicalsystem}.
\label{prop:limitpoint}
\end{proposition}

The reciprocal is not true, however, which is known to occur for \emph{unstable} equilibria of nonlinear systems. Therefore, we now focus on another key concept in dynamical systems theory -- that of \emph{(Lyapunov) stability} -- and proceed to study its relationship with convergence of the EM algorithm. 

\subsection{Lyapunov Stability}
\label{subsec:lyapunovstability}

We say that $\hat{\theta}^\star$ in the parameter space is a \emph{stable} point of the system~\eqref{eq:dynamicalsystem} if the trajectory $\{\hat{\theta}_0,\hat{\theta}_1,\hat{\theta}_2,\ldots\}$ is arbitrarily close to $\hat{\theta}^\star$, provided that $\|\hat{\theta}_0 - \hat{\theta}^\star\|>0$ is sufficiently small. In means that, if and only, for any $\varepsilon > 0$, there exists some $\delta>0$ such that for every $\hat{\theta}_0$ satisfying \mbox{$\|\hat{\theta}_0 - \hat{\theta}^\star\| \leq \delta$}, we have $\|\hat{\theta}_k - \hat{\theta}^\star\| \leq \varepsilon$ for every $k\in\mathbb{Z}_+$. If,~in addition, there exists some $\delta>0$ small enough such that, for every $\hat{\theta}_0$ satisfying $\|\hat{\theta}_0 - \hat{\theta}^\star\| \leq \delta$, we have $\hat{\theta}_k\to\hat{\theta}^\star$ as $k\to\infty$, then we say~$\hat{\theta}^\star$ is a \emph{(locally) asymptotically} stable point of the system. In other words, asymptotically stable points are simply stable and attractive points of the system. If the attractiveness is global, meaning that $\delta>0$ can be made arbitrarily large, then we say that~$\hat{\theta}^\star$ is \emph{globally} asymptotically stable. Finally, if in addition there exist $\delta,c>0$ and $\rho\in (0,1)$ such that, for every $\hat{\theta}_0$ satisfying $\|\hat{\theta}_0 - \hat{\theta}^\star\| \leq \delta$, we have $\|\hat{\theta}_k - \hat{\theta}^\star\| \leq c\,\rho^k$ for every $k\in\mathbb{Z}_+$, then we say that~$\hat{\theta}^\star$ is a \emph{(locally) exponentially} stable point of the system. Global exponential stability holds if $\delta>0$ can be made arbitrarily large.

\begin{lemma}
Consider the dynamical system~\eqref{eq:dynamicalsystem} with an arbitrary continuous function $F$. Then, every stable point is also an equilibrium. 
\label{lemma:stableareequilibria}
\end{lemma}

In particular, we can now see that, under Assumptions~\ref{ass:Markov}--\ref{ass:Qcontinuous}, \emph{every stable point of the dynamical system that represents EM is also a fixed point of EM}. Furthermore, it should be clear that local maxima of the incomplete-data log-posterior (that happen to be asymptotically stable) must be locally convergent points of EM. On the other hand, exponentially stable local maxima lead EM to attain a $\rho$-linear convergence rate, where $\rho$ originates from the definition of exponential stability. In particular, if the cost function~$\ell(\hat{\theta})$ is \mbox{Lipschitz} continuous, then clearly $\ell(\hat{\theta}_k) - \ell(\hat{\theta}^\star) = \mathcal{O}(\rho^k)$. The the relationships between all the aforementioned concepts is summarized in Appendix~\ref{appendix:stability}.

To establish the different notions of stability, we use the ideas proposed by Lyapunov, which we summarize in Lemma~\ref{lemma:lyapunovtheorem} in what we refer to as the ``Lyapunov theorem'' (actually a collective of results). Before stating the Lyapunov theorem, let us introduce some convenient terminology \mbox{borrowed} from nonlinear systems theory.
We say that a function $V(\hat{\theta})$ is:
\begin{enumerate}[noitemsep,topsep=0pt]
    \item \emph{positive semidefinite} w.r.t.~$\hat{\theta}^\star$ if $V(\hat{\theta}^\star) = 0$ and $V(\hat{\theta}) \geq 0$ for $\hat{\theta}$ near $\hat{\theta}^\star$;
    \item \emph{positive definite} w.r.t~$\hat{\theta}^\star$ if $V(\hat{\theta}) \geq 0$ with equality if and only if $\hat{\theta} = \hat{\theta}$, for every~$\hat{\theta}$ near~$\hat{\theta}^\star$;
    \item \emph{negative semidefinite} (respectively, \emph{negative definite}) if $-V$ is positive semidefinite (respectively, positive definite);
    \item \emph{radially unbounded} (or \emph{coercive}) if $V$ has domain $\mathbb{R}^p$ and $V(\hat{\theta})\to +\infty$ as $\|\hat{\theta}\|\to\infty$.
\end{enumerate}
Definitions 1--3 are local, and global reciprocals hold if $\hat{\theta}$ can be picked anywhere in the state space. Finally, we say that a scalar function $\alpha:[0,\infty)\to [0,\infty)$ is of \emph{class}~$\mathcal{K}$ if $\alpha(0)=0$ and if it is continuous and strictly increasing.

We are now ready to state Lyapunov's theorem for generic discrete-time systems of the form~\eqref{eq:dynamicalsystem}. These results, and Lyapunov-based results alike, are of crucial importance in practice for showing the different forms of stability of nonlinear systems.


\begin{lemma}[Lyapunov theorem]
Consider the dynamical system~\eqref{eq:dynamicalsystem} with an arbitrary continuous function $F$, and let $\hat{\theta}^\star$ be an arbitrary point in the state space. Let $V(\hat{\theta})$ be a continuous function and $\Delta V \triangleq V\circ F - V$. Consider the following assertions about these functions:
\begin{enumerate}[noitemsep,topsep=0pt]
    \item $V$ is positive definite w.r.t.~$\hat{\theta}^\star$;
    \item $\Delta V$ is negative semidefinite w.r.t.~$\hat{\theta}^\star$;
    \item $\Delta V$ is negative definite w.r.t.~$\hat{\theta}^\star$;
    \item $V$ and $-\Delta V$ are both globally positive definite w.r.t.~$\hat{\theta}^\star$ and $V$ is radially unbounded;
    \item There exist class-$\mathcal{K}$ functions $\alpha_1,\alpha_2,\alpha_3$ such that $\alpha_2(s) \leq \alpha_1(\rho\, s) + \alpha_3(s)$ holds near $s=0$, for some $\rho\in (0,1)$, and, for every $\hat{\theta}$ near $\hat{\theta}^\star$:
    \begin{subequations}
    \label{eq:DeltaVexp}
    \begin{align}
        \alpha_1(\|\hat{\theta}-\hat{\theta}^\star\|) \leq V(\hat{\theta}) &\leq \alpha_2(\|\hat{\theta}-\hat{\theta}^\star\|)\\
        \Delta V(\hat{\theta}) &\leq -\alpha_3(\|\hat{\theta}-\hat{\theta}^\star\|).
    \end{align}
    \end{subequations}
\end{enumerate}
\vskip-10pt
Then,~$\hat{\theta}^\star$ is
\begin{itemize}[noitemsep,topsep=0pt]
    \item stable if 1 and 2 hold;
    \item asymptotically stable if 1 and 3 hold;
    \item globally asymptotically stable if 4 holds;
    \item exponentially stable if 5 holds.
\end{itemize}
\label{lemma:lyapunovtheorem}
\end{lemma}
See Appendix~\ref{sec:lyapunovthmproof} for the proof of Lemma~\ref{lemma:lyapunovtheorem}. Notice that $\Delta V$ was designed so that $\Delta V_k = V_{k+1} - V_k$, where $\Delta V_k \triangleq \Delta V(\hat{\theta}_k)$ and $V_k \triangleq V(\hat{\theta}_k)$. Furthermore, it is intended to represent a discrete-time equivalent to $\dot{V}(\hat{\theta}(t)) = \frac{\mathrm{d}}{\mathrm{d}t}V(\hat{\theta}(t))$. The negative semidefiniteness simply translates to non-strict monotonicity of $\{V_k\}_{k\in\mathbb{Z}_+}$, which can be interpreted as a surrogate to $\{\ell(\hat{\theta}_k)\}_{k\in\mathbb{Z}_+}$.

\section{Main Results: Convergence of EM}
We now provide conditions that establish the different notions of Lyapunov stability explored thus far, for the dynamical system representation of the EM algorithm, and make appropriate conclusions in terms of the convergence of EM. For the missing proofs, please refer to the appendix. To proceed, we first propose the natural candidate for a Lyapunov function in optimization, the function
\begin{equation}
    V(\hat{\theta}) \triangleq \ell(\hat{\theta}) - \ell(\hat{\theta}^\star) =  \log p(\hat{\theta}^\star|y) - \log p(\hat{\theta}|y),
    \label{eq:Vcandidate}
\end{equation}
where $\hat{\theta}^\star$ is a particular strict local maximum of interest (fixed for the remaining of this section), \emph{i.e.} a particular point in the parameter space (state space) for which we seek convergence of EM. Since employing the Lyapunov theorem requires $V$ to be continuous, we make a mild assumption on the continuity of the incomplete-data posterior and then establish (non-asymptotic) stability.
\begin{assumption}
$\theta\mapsto\log p(\theta|y)$ is continuous.
\label{ass:posteriorcontinuous}
\end{assumption}
\begin{proposition}
\label{prop:stability}
Under Assumptions~\ref{ass:Markov}--\ref{ass:posteriorcontinuous}, any strict local maximum $\hat{\theta}^\star$ of the incomplete-data log-posterior \mbox{$\theta\mapsto \log p(\theta|y)$} is a stable equilibrium of the dynamical system~\eqref{eq:dynamicalsystem} that represents~EM.
\end{proposition}
\begin{proof}
Consider the candidate Lyapunov function~\eqref{eq:Vcandidate}, \mbox{defined} for~$\hat{\theta}$ near~$\hat{\theta}^\star$. Clearly, $V$ is positive definite w.r.t. $\hat{\theta}^\star$. Furthermore, since $Q(F(\hat{\theta}),\hat{\theta}) = \max_\theta Q(\theta,\hat{\theta})$, then $Q(F(\hat{\theta}),\hat{\theta})\geq Q(\hat{\theta},\hat{\theta})$. Plugging, $Q(\theta,\hat{\theta}) = \log p(\theta|y) - d(\theta,\hat{\theta})$ and rearranging terms, we find that $\Delta V(\hat{\theta}) = \log p(\hat{\theta}|y) -\log p(F(\hat{\theta})|y) \leq -d(F(\hat{\theta}),\hat{\theta}) \leq 0$. The result follows by Lyapunov's theorem.
\end{proof}

Clearly, to establish asymptotic stability, it suffices that $d(F(\theta),\theta)>0$ for $\theta\neq\hat{\theta}^\star$. In order to achieve this, we make the following assumption.
\begin{assumption}
$d(\theta,\hat{\theta}) = 0$ if and only if $\theta = \hat{\theta}$, for~$\hat{\theta}$ near~$\hat{\theta}^\star$ and arbitrary $\theta$.
\label{ass:strongidentifiability}
\end{assumption}
 This is also a relatively mild condition when~$\hat{\theta}^\star$ is a ``good'' local minimizer of $\ell(\theta)$. It follows, for instance, from a strong form of identifiability of the parameterized posterior latent distribution. It suffices that, for $\hat{\theta}_1\neq \hat{\theta}_2$ near~$\hat{\theta}^\star$, the conditional densities $p_X(\cdot|y,\hat{\theta}_1)$ and $p_X(\cdot|y,\hat{\theta}_2)$ differ with non-zero probability.
Failure to have Assumption~\ref{ass:strongidentifiability} be satisfied near a particular strict local maximizer~$\hat{\theta}^\star$ could result in that point not being being a fixed point of EM (\emph{i.e.} equilibrium of EM's dynamical system representation), let alone an asymptotically stable point and thus not locally convergent or even a limit point. 

Unfortunately, it is a well-documented behavior of the EM algorithm that its convergence properties and overall performance heavily rely on whether the initialization occurred near a ``good'' local maximizer of the incomplete-data log-likelihood or log-posterior~\citep{Figueiredo2004}.
In the context of GMM clustering, for instance, Assumption~\ref{ass:strongidentifiability} simply means that, after having sampled the GMM at hand, then the different posterior class probabilities will strictly depend on any perturbation to the parameters in the model (namely, the prior class probabilities and their corresponding means and covariance matrices).

\subsection{Local and Global Convergence of EM}
We now formally state and prove the local convergence of EM as a consequence of asymptotic Lyapunov stability. The proof is the straightforward culmination of previous section's discussion.

\begin{theorem}[Local Convergence]
Let $\hat{\theta}^\star$ a strict local maximizer of the incomplete-data posterior $\theta\mapsto p(\theta|y)$ and an isolated fixed point of EM. If Assumptions~\ref{ass:Markov}--\ref{ass:strongidentifiability} hold, then $\hat{\theta}^\star$ is an asymptotically stable equilibrium of the dynamical system~\eqref{eq:dynamicalsystem} that represents the EM algorithm~\eqref{eq:EMalg}, and thus EM is locally convergent to~$\hat{\theta}^\star$.
\label{thm:localconvergence}
\end{theorem}

\begin{proof}
We can reuse the argument in the proof of Proposition~\ref{prop:stability}, where the last inequality is strict for $\theta$ near $\hat{\theta}^\star$, due to $\hat{\theta}^\star$ being an isolated fixed point of EM and Assumption~\ref{ass:strongidentifiability}. Thus, $\Delta V$ is now negative definite  w.r.t.~$\hat{\theta}^\star$ instead of only semidefinite.
\end{proof}


In practice, since fixed points of EM must be stationary points of the log-posterior~\citep{Figueiredo2004}, then a sufficient condition for this assumption to hold would be the continuous differentiability of $\ell(\hat{\theta})$ for $\hat{\theta}$ near~$\hat{\theta}^\star$, and that $\hat{\theta}^\star$ were an isolated stationary point. We now state the conditions that lead to global convergence of EM.
\begin{theorem}
Suppose that 
$\{\theta: p(\theta|y) > 0\} = \mathbb{R}^p$
and that the conditions of Theorem~\ref{thm:localconvergence} hold, with Assumption~\ref{ass:strongidentifiability} holding globally (\emph{i.e.} $d(\theta,\hat{\theta}) = 0$ if and only if $\theta=\hat{\theta}$ for every $\theta,\hat{\theta}\in\mathbb{R}^p$). If $\ell(\hat{\theta}) \triangleq -\log p(\theta|y)$ is radially unbounded and~$\hat{\theta}^\star$ is the only fixed point of EM, then~$\hat{\theta}^\star$ is a  globally asymptotically stable equilibrium of the dynamical system~\eqref{eq:dynamicalsystem} that represents the EM algorithm~\eqref{eq:EMalg}, and thus EM is globally convergent to~$\hat{\theta}^\star$.
\label{thm:globalconvergence}
\end{theorem}
\begin{proof}
We can the argument in the proofs of Proposition~\ref{prop:stability} and Theorem~\ref{thm:localconvergence} to establish asymptotic stability. Furthermore, condition~4 of the Lyapunov theorem (radial unboundedness) clearly holds with the additional assumptions.
\end{proof}

In particular, the unicity of fixed points follows, for the case of continuously differentiable incomplete-data log-posterior, for unimodal distributions. Furthermore, we require the support of the posterior to be the entire Euclidean space of appropriate dimension, and to vanish radially, (\emph{i.e.} $p(\theta|y) \to 0$ as $\|\theta\|\to\infty$). These last two conditions are largely technical and can be roughly circumvented in practice (for \mbox{absolutely} continuous posterior distributions). On the other hand, the unimodality rarely occurs and, in fact, EM can easily converge to saddle points or diverge.

\subsection{Linear and Quadratic Convergence of EM}

Returning to the assumptions that lead to exponential stability, and thus local convergence of EM, we now explore alternatives that can lead to linear and quadratic convergence of EM. We achieve this by further strengthening the strong identifiability Assumption~\ref{ass:strongidentifiability}. In addition, we want condition~5 of the Lyapunov theorem to be satisfied for the candidate Lyapunov function~\eqref{eq:Vcandidate}.

\begin{assumption}
\label{ass:quadraticconv}
There exist some class-$\mathcal{K}$ functions $\alpha_1,\alpha_2,\alpha_3$ such that $\alpha_2(s) \leq \alpha_1(\rho\, s) + \alpha_3(s)$ holds near $s=0$, for some $\rho\in (0,1)$, and for every $\theta,\hat{\theta}$ with $\hat{\theta}^\star$ near $\hat{\theta}^\star$
\begin{subequations}
\label{eq:quadraticconv2}
\begin{align}
    \mathrm{e}^{-\alpha_2(\|\hat{\theta}-\hat{\theta}^\star\|)} \leq \frac{p(\hat{\theta}|y)}{p(\hat{\theta}^\star|y)} &\leq \mathrm{e}^{-\alpha_1(\|\hat{\theta}-\hat{\theta}^\star\|)},\\
    d(\theta,\hat{\theta}) &\geq \alpha_3(\|\hat{\theta}-\hat{\theta}^\star\|). 
\end{align}
\end{subequations}
\end{assumption}
This assumption sets the stage to state conditions for the linear convergence (in the sense of optimization) of the EM algorithm iterates.
\begin{theorem}[Linear convergence of iterates]
\label{thm:linConvIte}
Under the conditions of Theorem~\ref{thm:localconvergence} and Assumption~\ref{ass:quadraticconv}, $\hat{\theta}^\star$ is an exponentially stable equilibrium (with rate $\rho$) of the dynamical system~\eqref{eq:dynamicalsystem} that represents the EM algorithm~\eqref{eq:EMalg}, and thus $\|\hat{\theta}_k - \hat{\theta}^\star\| \leq \|\hat{\theta}_0 - \hat{\theta}^\star\|\cdot \rho^k = \mathcal{O}(\rho^k)$.
\end{theorem}

\begin{proof}
We can reuse the argument the argument of Proposition~\ref{prop:stability} to establish asymptotic stability. Furthermore, from~\eqref{eq:quadraticconv2}, we obtain~\eqref{eq:DeltaVexp}. Therefore, condition~5 of Lemma~\ref{lemma:lyapunovtheorem} is satisfied, and thus exponential stability is certified.
\end{proof}

Similarly to what was noted by~\cite{Taylor2018}, if $\ell(\hat{\theta})$ is Lipschitz continuous, then we clearly have $\ell(\hat{\theta}_k) - \ell(\hat{\theta}^\star) \leq L\|\hat{\theta}_0 - \hat{\theta}^\star\|\cdot \rho^k = \mathcal{O}(\rho^k)$. In that sense, $p(\hat{\theta}_k|y) \to p(\hat{\theta}^\star|y)$ actually converges \emph{quadratically}, in the sense of optimization. Alternatively, it would have sufficed that $(\alpha_3\circ \alpha_2^{-1})(s)\geq (1-\mu)\cdot s$ for some $\mu\in (0,1)$ to achieve $\ell(\hat{\theta}_k) - \ell(\hat{\theta}^\star) = \mathcal{O}(\mu^k)$. However, as discussed in Section~\ref{subsec:lyapunovstability}, we can relax the condition in Lyapunov's theorem required for exponential stability into something that leads to a notion of stability stronger than asymptotic stability but weaker than exponential stability, and which allows us to directly establish the Q-linear convergence of $\{V_k\}_{k\in\mathbb{Z}_+}$.

\begin{theorem}[Quadratic convergence of posterior]
\label{thm:quadposterior}
Under the conditions of Theorem~\ref{thm:localconvergence}, if there exists some $\mu\in (0,1)$ such that
\begin{equation}
    p(\hat{\theta}|y) \geq p(\hat{\theta}^\star|y)\cdot\mathrm{e}^{-\frac{d(F(\hat{\theta}),\hat{\theta})}{1-\mu}},
    \label{eq:quadposterior}
\end{equation}
for every $\hat{\theta}$ near $\hat{\theta}^\star$, then  $\lim\limits_{k\to\infty} \log p(\hat{\theta}_k|y) = \log p(\hat{\theta}^\star|y)$ with a Q-linear convergence rate upper bounded by~$\mu$. In particular, we have $\log p(\hat{\theta}^\star|y) - \log p(\hat{\theta}_k|y) \leq \log\left(\frac{p(\hat{\theta}^\star|y)}{p(\hat{\theta}_0|y)}\right)\mu^k = \mathcal{O}(\mu^k)$.
\end{theorem}

\begin{proof}
Once again reusing the argument of Proposition~\ref{prop:stability}, we have asymptotic stability, Furthermore, plugging $-d(F(\hat{\theta}),\hat{\theta})\geq \Delta V(\hat{\theta})$ into~\eqref{eq:quadposterior} and rearranging terms, we find that $V(F(\hat{\theta}))\leq \mu V(\hat{\theta})$. Thus, $V(\hat{\theta}_{k+1})\leq \mu\cdot V(\hat{\theta}_k)$, and therefore $V(\hat{\theta}_k) \leq V(\hat{\theta}_0)\cdot\mu^k = \mathcal{O}(\mu^k)$.
\end{proof}

Notice that~\eqref{eq:quadposterior} can be restated as~$d(F(\hat{\theta}),\hat{\theta}) \geq (1-\mu)\log\left(\frac{p(\hat{\theta}^\star|y)}{p(\hat{\theta}|y)}\right),$
for $\hat{\theta}$ near $\hat{\theta}^\star$, to more closely resemble~\eqref{eq:quadraticconv2}. Naturally, the main disadvantage of the last result is that checking the inequality~\eqref{eq:quadposterior} is likely to be virtually impossible in practice, given that it is directly based on the~EM dynamics $F$. Furthermore, it has been widely observed that EM's convergent rate is often sublinear, so the conditions in the previous results likely only hold in a few ``good'' local maximizers of the incomplete-data log-posterior.
The last result focus on linear convergence of the posterior, by using $V(\hat{\theta}) \triangleq p(\hat{\theta}^\star|y) - p(\hat{\theta}|y)$ as the new candidate Lyapunov function.
\begin{theorem}[Linear convergence of the posterior]
\label{thm:linConvPos}
Under the conditions of Theorem~\ref{thm:localconvergence}, if the concavity-like condition
\begin{equation}
    \label{eq:linConvPos}
    p(F(\hat{\theta})|y) \geq \mu\, p(\hat{\theta}) + (1-\mu)p(\hat{\theta}^\star|y),
\end{equation}
holds for every $\hat{\theta}$ near~$\hat{\theta}^\star$, then $p(\hat{\theta}^\star|y) - p(\hat{\theta}_k|y) \leq (p(\hat{\theta}^\star|y) - p(\hat{\theta}_0|y))\cdot \mu^k = \mathcal{O}(\mu^k)$.
\end{theorem}

\begin{proof}
This time, we consider the continuous and positive definite candidate Lyapunov function $V(\hat{\theta}) \triangleq p(\hat{\theta}^\star|y) - p(\hat{\theta}|y)$. By rearranging terms in~\eqref{eq:linConvPos}, we can show that $V(F(\hat{\theta}))\leq \mu\cdot V(\hat{\theta})$.
\end{proof}

\subsection{Experimental validation of the convergence properties}
We demonstrate the convergence properties of the MAP-EM algorithm on a general GMM with independent Gaussian priors on the unknown means. The details can be found in Appendix~\ref{sec:experiments}, where we can see that, as the prior becomes more informative, convergence is achieved at faster rate.  We also note that our convergence results readily apply to the MAP-EM algorithm over many distributions other than GMMs, and thus the wide range of applications it has been applied to. 

\begin{figure}[!ht]
    \centering
    \includegraphics[width = 0.4\linewidth]{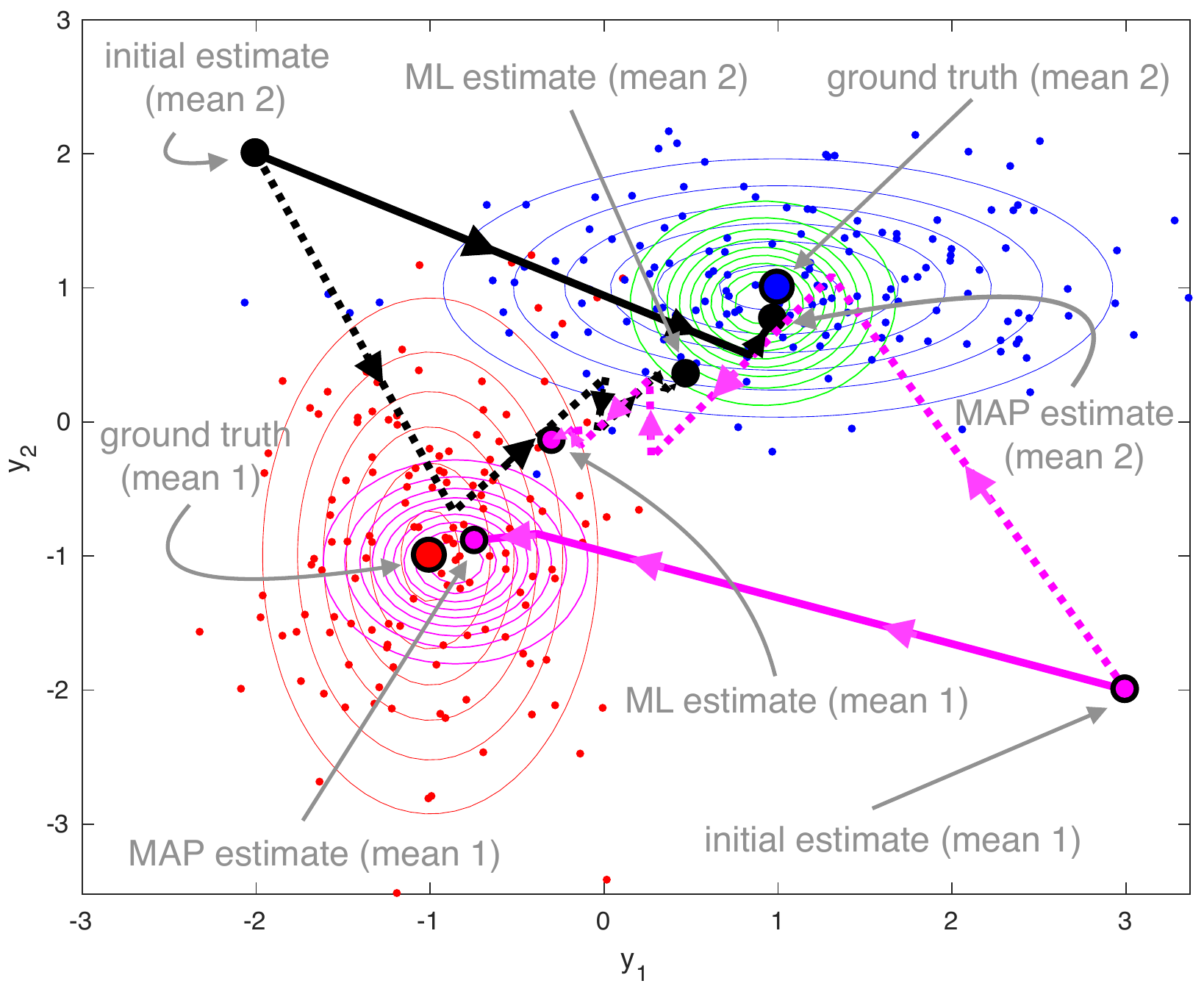}  ~~
    \includegraphics[width = 0.55\linewidth]{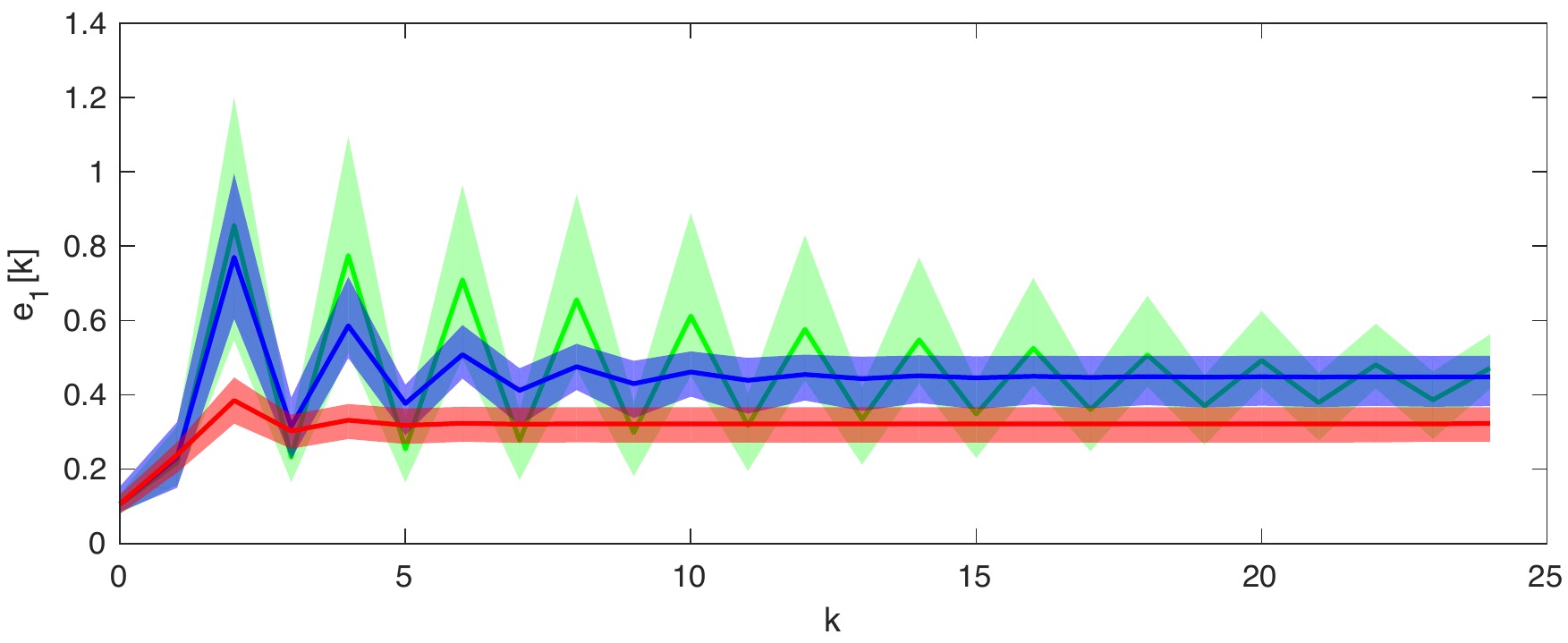}
    \caption{\small ({\bf left}) GMM with Gaussian priors on the means. The components 1 and 2 of the GMM are depicted, respectively, by the colors red and blue, and their corresponding EM estimate are depicted, respectively, by the colors magenta and green. The dashed line represents EM with a flat prior, whereas the full line with the Gaussian prior with contour lines depicted by the corresponding color. 
    ({\bf right}) Rate $\mu_1[k] \triangleq \frac{\|\hat{\theta}_1[k+1] - (\hat{\theta}_\mathrm{MAP})_1\|}{\|\hat{\theta}_1[k] - (\hat{\theta}_\mathrm{MAP})_1\|}$ for progressively more informative priors: $\sigma_{m,0} = 0.15,\,0.1,\,0.05$ depicted, respectively, by the colors green, blue, and red.  The linear convergence rates roughly appear to be $\mu_1=0.43,\,0.41,\,0.31$, respectively.}
    \label{fig:GMM_convrate}
\end{figure}

\section{Conclusion and Next Steps}
In this paper, we addressed a gap in analyzing and designing optimization algorithms from the point of view of dynamical systems and control theory. Indeed, m of the recent recent literature largely focus on \emph{continuous-time} representations (via ordinary differential equations or inclusions) of general-purpose \emph{gradient-based} nonlinear optimization algorithms. However, we provide a unifying framework for the study of iterative optimization algorithms as \emph{discrete-time} dynamical systems, and describe several relationships between forms of Lyapunov stability of state-space dynamical systems and convergence of optimization algorithms. In particular, we explored how exponential stability can be used to derive linear (or superlinear) convergence rates. 
We then narrowed this framework in detail to analyze convergence of the expectation-maximization (EM) algorithm for maximum a posteriori (MAP) estimation. Following first principles from dynamical systems stability theory, conditions for convergence of MAP-EM were developed, including conditions to show \emph{fast} convergence (linear or quadratic), though EM often converges sublinearly. 

The conditions we derived have a convenient statistical and information-theoretic interpretation and may thus be used in the future to design other novel EM-like algorithms with provable convergence rate. The conditions we derive would have been difficult to unveil without our approach, and thus we argue that a treatment similar to ours can we adopted for the convergence analysis of many other algorithms in machine learning. For future work, we believe that our approach can prove valuable in the design of EM-like algorithms for online estimation subject to a concept drift or data poisoning. In fact, by carefully \emph{designing} a controller (input) on an otherwise unstable system, we can leverage a process known as \emph{stabilization} to force different notions stability of dynamical systems that represent an optimization algorithm, and thus improve its convergence rate or robustness.


\newpage
\pagebreak

\bibliography{references}

\newpage
\pagebreak


\appendix
\section{Convergence via Subexponential Stability}
\label{appendix:Subexponential}
As is often remarked when addressing in the context of EM, monotonicity is generally not enough to attain convergence. Nevertheless, the negative definiteness condition ensures asymptotic stability, and thus (local) convergence. Indeed, from an optimization perspective, the Lyapunov conditions for exponential stability are simply equivalent to \emph{strict} monotonicity of $\{V_k\}_{k\in\mathbb{Z}_+}$, together with an absolute attainable lower bound of the surrogate $V(\hat{\theta})$ of $\ell(\hat{\theta})$ at $\hat{\theta} = \hat{\theta}^\star$.

Note that from~\eqref{eq:DeltaVexp} it follows that
\begin{equation}
    \Delta V(\hat{\theta}) \leq -(\alpha_3\circ\alpha_2^{-1})(V(\hat{\theta}))
    \label{eq:linearconvV}
\end{equation}
for every~$\hat{\theta}$ near $\hat{\theta}^\star$, which can be restated as
\begin{equation}
    V_{k+1} \leq (\mathrm{id}-\alpha_3\circ\alpha_2^{-1})(V_k),
    \label{eq:linearconvV2}
\end{equation}
for $\hat{\theta}_0$ near $\hat{\theta}^\star$, where $\mathrm{id}(s)\triangleq s$. Therefore, $V_k \leq (\mathrm{id}-\alpha_3\circ\alpha_2^{-1})^k(V_0)$, where $\alpha^0 \triangleq \mathrm{id}$ and $\alpha^{k+1} \triangleq \alpha \circ \alpha^k$.  In particular, if $\alpha_i (s) = a_i \, s^p$ with \mbox{$a_i > 0$} ($i=1,2,3$) and $p>0$, then~\eqref{eq:linearconvV} and~\eqref{eq:linearconvV2} become
\begin{equation}
    \label{eq:DeltaV1}
    \Delta V_k \leq -(1-\mu)V_k 
\end{equation}
and
\begin{equation}
    \label{eq:DeltaV2}
    V_{k+1} \leq \mu \, V_k,
\end{equation}
respectively, where $\mu \triangleq 1 - a_3/a_2 \in [0,1)$. In that case, we have $V_k\to 0$ as $k\to\infty$, with a \mbox{Q-linear} convergence rate \mbox{upper} bounded by~$\mu$. In particular, we have \mbox{$V_k \leq V_0 \cdot \mu^k = \mathcal{O}(\mu^k)$}. We further note that $\alpha_1$ and $\rho$ do not directly influence the bound $\mu$. This observation is similar in spirit to Lemma~1 in~\cite{Aitken1994}.

With these remarks into consideration, we note that if $V$ is continuous, positive definite w.r.t.~$\hat{\theta}^\star$, and ~\eqref{eq:DeltaV1} (equivalently,~\eqref{eq:DeltaV2}) holds for $\hat{\theta}$ near $\hat{\theta}^\star$, then the linear rate $V_k = \mathcal{O}(\mu^k)$ still holds, without necessarily having linear convergence of $\{\hat{\theta}_k\}_{k\in\mathbb{Z}_+}$. Therefore,~\eqref{eq:linearconvV} without necessarily~\eqref{eq:DeltaVexp} induces a weaker notion than exponential stability, but stronger than asymptotic stability. In fact, it coincides with the notion of \mbox{$\ell_p$-stability} for $\alpha_i(s) = a_i\,s^{p^i}$, with $p=p_3/p_2\geq 1$~\cite{Lakshmikantham2002}. 

From an optimization perspective, we can leverage the ideas discussed in the previous paragraphs by noting that we are often concerned about the convergence rate in terms of the objective function or some meaningful surrogate of it, rather than directly the iterates in a numerical optimization scheme. This approach was implicit, for instance, in~\cite{Taylor2018,Vaquero2019}.

\section{Relationships between stability and convergence}
\label{appendix:stability}
We represent the relationships between all the aforementioned concepts (local variants only, for the sake of simplicity) through the following diagram:

\begin{center}
{
\begin{tikzcd}[arrows=Rightarrow]
& \textnormal{exponentially stable} \arrow[d]\arrow[r]       & \textnormal{linearly convergent} \arrow[d]\\
& \textnormal{asymptotically stable} \arrow[d]\arrow[r]        & \textnormal{locally convergent}  \arrow[d]\\
& \textnormal{stable point}            \arrow[d]                 & \textnormal{limit point}        \arrow[d]\\
& \textnormal{equilibrium}           \arrow[r, Leftrightarrow] & \textnormal{fixed point.}
\end{tikzcd}
}
\end{center}

\section{Proof of Proposition~\ref{prop:infoEM}}
From Assumption~\ref{ass:Markov}, we have
\begin{equation}
    p(y|x) = p(y|x,\theta) = \frac{p(y|\theta)}{p(x|\theta)}p(x|y,\theta) = \frac{1}{p(\theta|x)}\frac{p(y)}{p(x)}p(\theta|y)p(x|y,\theta),
\end{equation}
and therefore
\begin{equation}
    p(\theta|x) = \frac{p(y)/p(x)}{p(y|x)}p(\theta|y)p(x|y,\theta) = \frac{p(\theta|y)p(x|y,\theta)}{p(x|y)}.
\end{equation}

Let $f_1(\theta)\sim f_2(\theta)$ indicate that $\theta\mapsto f_1(\theta)- f_2(\theta)$ is constant. Then, we have
\begin{equation}
    \log p(\theta |x) \sim \log p(\theta|y) + \log p(x|y,\theta) = \log p(\theta|y)  - \log\left(\frac{p(x|y,\hat{\theta})}{p(x|y,\theta)}\right)
\end{equation}
for any fixed~$\hat{\theta}$. Taking the expected value in $x\sim p(x|y,\hat{\theta})$ and attending to the definitions of EM~\eqref{eq:EMalg} and KL divergence (Subsection~\ref{subsec:info_theory}), then the expression~\eqref{eq:Qtemp} follows. Finally,~\eqref{eq:EMnew} readily follows by combining~\eqref{eq:EMalg} and~\eqref{eq:Qtemp} and disregarding any additive terms that do not depend on~$\theta$, since those will not affect the so-called \mbox{\emph{M-step}} (maximization step) of the EM algorithm, \emph{i.e.} the maximization in~\eqref{eq:EMalg}. $\hfill\blacksquare$

\section{Proof of Lemma~\ref{lemma:Fiscontinuous}}
\label{sec:lemma2proof}
Let $\{\hat{\theta}_k\}$ a convergent sequence, not necessarily generated by the EM algorithm (\emph{i.e.} without necessarily having $\hat{\theta}_{k+1} = F(\hat{\theta}_k)$), with limit $\lim\limits_{k\to\infty}\hat{\theta}_k = \hat{\theta}$. Notice that
\begin{subequations}
\begin{align}
    Q\left(\lim_{k\to\infty} F(\hat{\theta}_k),\hat{\theta}\right) &= Q\left(\lim_{k\to\infty}F(\hat{\theta}_k),\lim_{k\to\infty} \hat{\theta}_k\right)\\
    &= \lim_{k\to\infty} Q(F(\hat{\theta}_k),\hat{\theta}_k) \label{eq:cont1} \\
    &\geq \lim_{k\to\infty} Q(F(\hat{\theta}),\hat{\theta}_k) \label{eq:ineq1} \\
    &= Q\left(F(\hat{\theta}),\lim_{k\to\infty}\hat{\theta}_k\right) \label{eq:cont2} \\
    &= Q(F(\hat{\theta}),\hat{\theta})\\
    &\geq Q\left(\lim_{k\to\infty} F(\hat{\theta}_k),\hat{\theta}\right), \label{eq:ineq2} 
\end{align}
\end{subequations}
where~\eqref{eq:cont1} and~\eqref{eq:cont2} both follow from the continuity of the \mbox{$Q$-function} (Assumption~\ref{ass:Qcontinuous}). On the other hand, the inequalities~\eqref{eq:ineq1} and~\eqref{eq:ineq2} follow by noting that $Q(F(\hat{\theta}_k),\hat{\theta}_k) \geq F(\theta,\hat{\theta}_k)$ and $Q(F(\hat{\theta}),\hat{\theta}) \geq Q(\theta,\hat{\theta})$ for every $\theta$. In particular, they follow by choosing $\theta = F(\hat{\theta})$ and $\theta = \lim\limits_{k\to\infty} F(\hat{\theta}_k)$, respectively, and taking the limit $k\to\infty$ on both sides. 

Therefore, we have
\begin{equation}
    Q\left(\lim\limits_{k\to\infty} F(\hat{\theta}_k),\hat{\theta}\right) = Q(F(\hat{\theta}),\hat{\theta}) =  \max\limits_\theta Q(\theta,\hat{\theta}),
\end{equation}
and thus $\lim\limits_{k\to\infty} F(\hat{\theta}_k) \in \argmax\limits_\theta Q(\theta,\hat{\theta}) = \{F(\hat{\theta})\}$, which in turn makes $F$ continuous. $\hfill\blacksquare$

\section{Proof of Proposition~\ref{prop:limitpoint}}
Let $\hat{\theta}_0$ be such that $\lim\limits_{k\to\infty}\hat{\theta}_k = \hat{\theta}^\star$, where $\{\hat{\theta}_k\}_{k\in\mathbb{Z}_+}$ was generated~\eqref{eq:EMalg}. Then,
\begin{equation}
    F(\hat{\theta}^\star) = F\left(\lim_{k\to\infty}\hat{\theta}_{k}\right) = \lim_{k\to\infty} F(\hat{\theta}_{k}) = \lim_{k\to\infty} \hat{\theta}_{k+1} = \hat{\theta}^\star,
\end{equation}
where the second equality follows from the continuity of~$F$. $\hfill\blacksquare$

\section{Proof of Lemma~\ref{lemma:stableareequilibria}}
Let $\varepsilon^{(n)} > 0$ such that $\varepsilon^{(n)}\to 0$ as $n\to\infty$ and let $\delta^{(n)} > 0$ be such that, if $\|\hat{\theta}_0 - \hat{\theta}^\star\| \leq \delta^{(n)}$, then $\|\hat{\theta}_k - \hat{\theta}^\star\| \leq \varepsilon^{(n)}$ for every $k\in\mathbb{Z}$. Naturally, $0 < \delta^{(n)}\leq \varepsilon^{(n)}$, and thus $\delta^{(n)}\to 0$ as $n\to\infty$. Therefore, given any sequence $\{\hat{\theta}_0^{(n)}\}_{n\in\mathbb{Z}_{+}}$ such that $\|\hat{\theta}_0^{(n)} - \hat{\theta}^\star\| \leq \delta^{(n)}$, then $\hat{\theta}_0^{(n)} \to \hat{\theta}^\star$ as $n\to\infty$. Furthermore, $\|F(\hat{\theta}_0^{(n)}) - \hat{\theta}^\star\| = \|\hat{\theta}_1^{(n)} - \hat{\theta}^\star\| \leq \varepsilon^{(n)}$, and thus $F(\hat{\theta}_0^{(n)})\to\hat{\theta}^\star$ as $n\to\infty$. From the continuity of $F$, it follows that $F(\hat{\theta}^\star) = F\left(\lim\limits_{n\to\infty}\hat{\theta}_0^{(n)}\right) = \lim\limits_{n\to\infty} F(\hat{\theta}_0^{(n)}) = \hat{\theta}^\star$. $\hfill\blacksquare$

\section{Proof of Lemma~\ref{lemma:lyapunovtheorem} (Lyapunov Theorem)}
\label{sec:lyapunovthmproof}
For the non-asymptotic and asymptotic stability part of the Lyapunov theorem, please refer to the proofs of Theorems~5.9.1--5.9.2 in~\cite{Agarwal2000}, Corollary~4.8.1 and Theorem~4.8.3 in~\cite{Lakshmikantham2002}, or Theorem~1.2 in~\cite{Bof2018}. For the global asymptotic stability, please refer to Theorem~5.9.8 in~\cite{Agarwal2000}, Theorem~4.9.1. in~\cite{Lakshmikantham2002}, or Theorem~1.4 in~\cite{Bof2018}.

For the exponential stability, we can adapt the ideas in~\cite{Aitken1994} and the proof of Theorem~5.4 in~\cite{Bof2018}. More precisely, we have
\begin{equation}
    \alpha_1(\|\hat{\theta}_{k+1}-\hat{\theta}^\star\|) \leq V_{k+1} = \Delta V_k + V_k \leq \alpha_2(\|\hat{\theta}_k-\hat{\theta}^\star\|) - \alpha_3(\|\hat{\theta}_k-\hat{\theta}^\star\|) \leq \alpha_1(\rho\|\hat{\theta}_k-\hat{\theta}^\star\|),
\end{equation}
and thus $\|\hat{\theta}_{k+1}-\hat{\theta}^\star\| \leq \rho\|\hat{\theta}_k-\hat{\theta}^\star\|$. In other words, $\hat{\theta}_k\to\hat{\theta}^\star$ as $k\to\infty$ Q-linearly, with a rate upper bounded by $\rho$. $\hfill\blacksquare$

\section{An Illustrative Example \& Experiments}
\label{sec:experiments}

Let $y=\{y^{(1)},\ldots,y^{(n)}\}$ be some dataset that we wish to cluster. Let us assume that $y^{(i)}\in\mathbb{R}^d$ is randomly, but independently, sampled from an unknown class $z^{(i)}\in\{1,\ldots,M\}$, which we thus seek to infer. For the sake of simplicity, consider that the class probabilities $\mathbb{P}(z^{(i)}=m) = \alpha_m\in (0,1)$ such that $\alpha_1 + \ldots + \alpha_M = 1$ are known. Furthermore, suppose that each class $m$ is associated with a Gaussian distribution $\mathcal{N}(\theta_m,\Sigma_m)$ with known $d\times d$ covariance matrix~$\Sigma_m\succ 0$, but unknown mean~$\theta_m\in\mathbb{R}^d$. Therefore, our data follows the Gaussian mixture model (GMM),
\begin{equation}
    p(y^{(i)}|\theta) = \sum_{m=1}^M \alpha_m \phi_m(y^{(i)}|\theta_m),
\end{equation}
with $\theta = \{\theta_1,\ldots,\theta_M\}$, where $\phi_m(\cdot|\theta_m)$ denotes the PDF of~$\mathcal{N}(\theta_m,\Sigma_m)$. 

In order to cluster our data, we seek to infer~$z^{(i)}$ for each datum~$y^{(i)}$. To achieve this, it is customary to first estimate~$\theta$ as some~$\hat{\theta}$, and then proceed by maximizing the class-conditional posterior:
\begin{subequations}
\begin{align}
    \mathrm{cluster}(y^{(i)}) &\triangleq \argmax_{m\in\{1,\ldots,M\}} \mathbb{P}(z^{(i)}=m|y^{(i)},\hat{\theta}) = \argmin_{m\in\{1,\ldots,M\}} \|y^{(i)} - \hat{\theta}_m\|_{\Sigma_m^{-1}},
\end{align}
\end{subequations}
where $\|y^{(i)} - \hat{\theta}_m\|_{\Sigma_m^{-1}} \triangleq \sqrt{(y^{(i)} - \hat{\theta}_m)^\top\Sigma_m^{-1}(y^{(i)} - \hat{\theta}_m)}$. In other words, the clusters are assigned by projecting each datum onto the class with smallest Mahalanobis distance.

We seek the MAP estimator or some otherwise ``good'' local maximizer~\cite{Figueiredo2004} of the log-posterior,
\begin{align}
\theta \mapsto\log p(\theta|y)
\sim \log p(\theta) +\sum_{i=1}^n\log\left(\sum_{m=1}^M\alpha_m\phi_m(y^{(i)}|\theta_m)\right),
\end{align}
where~$p(\theta) = p(\theta_1,\ldots,\theta_M)$ is some given prior. Unfortunately, stationary points cannot be analytically computed in general. Nevertheless, the EM algorithm can be derived for certain priors. Indeed, ignoring terms that do not depend on $\theta$ by fixing~$\hat{\theta}[k]=\{\hat{\theta}_1[k],\ldots,\hat{\theta}_m[k]\}$, we have
\begin{small}
\begin{align}
    & Q(\theta,\hat{\theta}[k]) \sim \log p(\theta_1,\ldots,\theta_M) -\frac{1}{2}\sum_{i=1}^n\sum_{m=1}^M\hat{\alpha}_m^{(i)}[k]\|y^{(i)} -  \theta_m\|_{\Sigma_m^{-1}}^2, \\
     & \hat{\alpha}_m^{(i)}[k] \triangleq \mathbb{P}(z^{(i)}=m|y^{(i)},\hat{\theta}[k])
    =  \frac{\alpha_m\phi_m(y^{(i)}|\hat{\theta}_m[k])}{\sum\limits_{m'=1}^M \alpha_{m'}\phi_{m'}(y^{(i)}|\hat{\theta}_{m'}[k])}. \nonumber
\end{align}
\end{small}
%
%
Notice that the prior may be seen as a regularization term for the M-step of the EM algorithm. However, the tractability of the M-step, and thus the EM algorithm itself, will strongly depend on the choice for the prior\footnote{For further discussion of MAP estimation and choice of priors for GMMs and other finite mixture models (FMM), as well as applications of FMMs in unsupervised learning and an alternative to EM, see~\cite{Figueiredo2002}.}. In particular, if we adopt independent priors $p(\theta_1,\ldots_M) = \prod_{m=1}^M p(\theta_m)$ from Gaussian or Laplacian distributions, then the M-step is mathematically tractable. Indeed, with these choices of priors, the M-step can be cast, respectively, as a Ridge a LASSO regression problem. Nevertheless, notice that the prior will change the MAP estimate, and thus a poor choice may negatively bias the estimation procedure.
In order to derive explicit expressions, we now focus on the case of independent Gaussian priors $\theta_m \overset{\mathrm{indep}}{\sim}~\mathcal{N}(\theta_{m,0},\Sigma_{m,0})$ with tunable parameters $\theta_{m,0}$~and $\Sigma_{m,0}$ ($m=1,\ldots,M$). From the previous discussion, we thus have
\begin{small}
\begin{align}
    \hat{\theta}_m[k+1] &=  \left(\Sigma_m\Sigma_{m,0}^{-1} + \sum_{i=1}^n \hat{\alpha}_m^{(i)}[k]\,\mathbb{I}_{d\times d}\right)^{-1} \cdot\left(\Sigma_m\Sigma_{m,0}^{-1}\theta_{m,0} + \sum_{i=1}^n\hat{\alpha}_m^{(i)}[k]y^{(i)}\right)
\end{align}
\end{small}
for $m\in\{1,\ldots,M\}$ and $k\in\mathbb{Z}_+=\{0,1,2,\ldots\}$. Naturally, the MAP-EM reduces to the standard EM for maximum likelihood (ML) estimation when the prior becomes flat (\emph{e.g.}, fixed arbitrary $\theta_{m,0}\in\mathbb{R}^d$ and $\Sigma_{m,0}^2 = \sigma_0^2\mathbb{I}_{d\times d}$ with $\sigma_0^2\to\infty$). Unlike the non-Bayesian case, the (known) covariance matrices~$\Sigma_1,\ldots,\Sigma_M$ will indeed influence the MAP-EM algorithm.

We then independently simulate $n=300$ points from a \mbox{$(d=2)$-dimensional} GMM with $M=2$ components and class probabilities $\alpha_1 = \alpha_2 = 1/2$. The true means will be placed at $\theta_1 = [-1,-1]^\top$ and $\theta_2 = [1,1]^\top$, and covariance matrices will be $\Sigma_1 = \mathrm{diag}(0.25,1)$ and $\Sigma_2 = \mathrm{diag}(1,0.25)$. We will use independent  Gaussian priors with covariance matrices $\Sigma_{1,0} = \Sigma_{1,0} = \Sigma_{2,0} = \mathrm{diag}(\sigma_0^2,\sigma_0^2)$ and means $\theta_{m,0}$ sampled from $\mathcal{N}(\theta_m,\Sigma_{m,0})$ for $m=1,2$. We will perform $T=20$ trials for each choice of $\sigma_0^2>0$. Finally, we place the initial estimates at $\hat{\theta}_1 = [3,-2]^\top$ and $\hat{\theta}_2 = [-2,2]^\top$.

The results are illustrated in Figures~\ref{fig:GMM_convrate} (left) and (right). 
We see that choosing a flat prior leads EM to converge to the ML estimate, which in this case is significantly farther from the true placement of the unknown means, compared to a non-flat prior with a highly informative prior. However, this is partly only true since, in our setup, $\sigma_0\to 0$ will lead the prior to become $p(\theta) = \delta(\theta-\theta_\mathrm{true})$. Furthermore, we demonstrate that, as the prior becomes more informative, convergence is achieved at  faster rate. The convergence rate appears to decrease and possibly approach superlinearity (recall that $\hat{\theta}[k]\to\hat{\theta}^\star$ superlinearly if $\lim_{k\to\infty} \|e[k+1]\|/\|e[k]\| = 0$, where $e[k] = \hat{\theta}[k] - \hat{\theta}^\star$). However, due to numerical stability issues, it is difficult to estimate the


\end{document}